\documentclass{article}

\usepackage{arxiv}

\usepackage[utf8]{inputenc} 
\usepackage[T1]{fontenc}    
\usepackage{hyperref}       
\usepackage{url}            
\usepackage{booktabs}       
\usepackage{amsfonts}       
\usepackage{nicefrac}       
\usepackage{microtype}      
\usepackage{amsmath, amsfonts, amssymb}
\usepackage{amsthm}
\usepackage{cleveref}       
\usepackage{graphicx}
\usepackage[numbers]{natbib}
\usepackage{doi}
\usepackage{multirow}
\usepackage{placeins}

\newtheorem{proposition}{Proposition}

\title{Neural Prior Estimation: Learning Class Priors from Latent Representations}
\date{}

\author{ 
\href{https://orcid.org/0009-0007-9525-0617}{\includegraphics[scale=0.06]{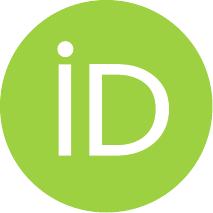}\hspace{1mm}Masoud Yavari}\\
Independent Researcher\\
Isfahan, Iran\\
\texttt{masoud.yavari@gmail.com} \\
\And
\href{https://orcid.org/0000-0001-7891-293X}{\includegraphics[scale=0.06]{orcid.pdf}\hspace{1mm}Payman Moallem} \\
Department of Electrical Engineering\\
Faculty of Engineering\\
University of Isfahan, Iran\\
\texttt{p\_moallem@eng.ui.ac.ir}
}

\renewcommand{\shorttitle}{Neural Prior Estimation}

\hypersetup{
hidelinks,
pdftitle={Neural Prior Estimation: Learning Class Priors from Latent Representations},
pdfauthor={Masoud Yavari, Payman Moallem},
pdfkeywords={long-tailed recognition, class imbalance, image classification, semantic segmentation, logit adjustment, prior estimation},
}

\begin{document}
\maketitle

\begin{abstract}
Logit adjustment corrects class imbalance using the empirical class prior. We study whether a comparable class-frequency signal can instead be learned from the network representation, without explicitly supplying class counts to the correction rule. We introduce the Neural Prior Estimator (NPE), which attaches one or more lightweight Prior Estimation Modules (PEMs) to the latent representation. Each PEM is trained with a one-way logistic objective on the ground-truth coordinate. The resulting frequency-dependent outputs are combined into an NPE estimate and used as a learned logit correction, forming NPE-LA. In a simplified scalar model, the optimum of the PEM objective is monotone in the class count and grows asymptotically as $\log N_c$, up to a slower $\log\log N_c$ term. Experiments on long-tailed CIFAR-10 and CIFAR-100 show that NPE-LA is competitive with standard logit adjustment and improves minority-class performance over CE and cRT in the reported settings. Experiments on STARE and ADE20K further show that the same idea can be used as a lightweight recalibration mechanism for dense prediction.
\end{abstract}

\keywords{Long-tailed recognition \and Class imbalance \and Image classification \and Semantic segmentation \and Logit adjustment \and Prior estimation}

\section{Introduction}
\label{sec:introduction}

Imbalanced datasets, where a small number of classes account for most of the training samples while many others are represented by only a few examples, remain a common difficulty in visual recognition~\cite{zhang2025systematic,zhang2023deep}. A classifier trained on such a distribution is repeatedly updated by examples from the head classes and can consequently develop decision boundaries that favor them, while the tail classes receive weaker and less stable supervision. Long-tailed recognition has therefore been studied from several directions, including data rebalancing, modified objectives, representation learning, classifier redesign, and multi-expert architectures.

One group of methods attempts to increase or redistribute the information available to minority classes. Transfer-based approaches use knowledge learned from frequent categories to improve the representations of rare ones~\cite{yin2018feature}, while feature-level augmentation methods synthesize or transfer class-dependent variations to enrich the tail~\cite{kim2020m2m,wang2021rsg,li2021metasaug}. Other approaches act directly on the effective training distribution. Classical over- and under-sampling change the frequency with which classes are observed, and more adaptive schemes use training dynamics or class performance to determine the sampling policy~\cite{chawla2002smote,wang2019dynamic,zang2021fasa}. At the objective level, class-balanced weighting, focal loss, Balanced Softmax, gradient corrections, and margin-based losses such as LDAM modify the contribution or decision margin of individual classes~\cite{cui2019class,lin2017focal,ren2020balanced,tan2020equalization,hsieh2021droploss,wang2021adaptive,cao2019learning}. These methods address the same imbalance from different points in the training pipeline, but they generally require the imbalance to be encoded explicitly through sampling probabilities, class counts, margins, or related statistics.

A related line of work focuses on the classifier and the learned representation. Decoupled methods such as cRT separate representation learning from classifier re-training, allowing the feature extractor to be learned on the natural distribution and the classifier to be re-estimated under a more balanced one~\cite{kang2019decoupling}. Other methods modify the classifier geometry or decision rule through normalization, hierarchical structures, or causal formulations~\cite{ye2020identifying,wu2020solving,tang2020long}. Multi-branch and multi-expert approaches instead learn complementary prediction functions. BBN, for example, uses two branches together with cumulative learning to balance representation and classifier behavior~\cite{zhou2020bbn}, while other expert-based methods specialize different predictors to different regions of the class distribution or encourage diversity among them~\cite{li2020overcoming,wang2020long,zhang2022self}. These approaches demonstrate that imbalance is reflected not only in the final logits but also in the representation and optimization dynamics that produce them.

Logit Adjustment (LA)~\cite{menon2020long} provides a particularly simple distribution-level correction. For a classifier trained on an imbalanced distribution, LA shifts the logits according to the logarithm of the empirical class prior. The adjustment follows directly from the prior term in Bayes' rule and can be applied without changing the backbone architecture. The same prior-dependent correction can also be incorporated into the training objective through the Logit Adjustment Loss. In both cases, however, the correction is constructed from class-frequency statistics supplied from outside the network.

This dependence on explicit counts is usually harmless on a fixed benchmark, where the complete training distribution is known in advance. It is nevertheless natural to ask whether the network itself contains a usable frequency signal. Feature representations are shaped by the number of observations of each class, by stochastic optimization, and by the geometry learned for individual categories. The empirical histogram and the structure encoded by the representation are therefore related, but they are not the same object. A learned estimate extracted from the representation could provide a different view of the frequency information that ultimately appears in the classifier decision rule. It could also be useful in settings where the observed distribution changes over time, although the present work does not evaluate such settings directly.

Previous calibration methods partly relax the reliance on a fixed prior by learning transformations of the classifier output. Posterior Re-calibration and DisAlign, for example, learn adaptive corrections that improve balanced prediction~\cite{tian2020posterior,zhang2021distribution}. Their objective, however, is to calibrate or align the prediction distribution rather than to expose an explicit prior-related signal learned from the latent representation. This distinction motivates the question studied here: can a class-frequency signal suitable for logit adjustment emerge directly from intermediate features, without providing class counts to the correction rule?

We introduce the Neural Prior Estimator (NPE) to investigate this question. NPE attaches one or more lightweight Prior Estimation Modules (PEMs) to the backbone representation. A PEM produces one output per class and is optimized with a one-way logistic objective evaluated only at the coordinate of the ground-truth class. Because that coordinate is updated whenever the corresponding class is observed, the accumulated optimization signal is frequency dependent. The resulting PEM output is interpreted as a learned surrogate for the class-frequency component of the logits. The outputs of the PEMs are combined into the NPE estimate, which can then be subtracted from the main classifier logits to form NPE-LA.

A linear PEM has the same input-output shape as an auxiliary classifier head, but the similarity is architectural rather than functional. It is not trained with softmax cross-entropy to provide an additional class prediction. Its objective is one-way and class-wise, and its output is used to estimate the frequency-related term employed by the subsequent logit correction. This separates NPE from dual-branch methods such as BBN, where the branches are themselves classifiers and are designed to learn complementary decision functions under a cumulative training strategy~\cite{zhou2020bbn}. A single PEM is sufficient to construct NPE; using multiple PEMs is an optional extension that changes the auxiliary gradient received by the shared backbone and is examined experimentally.

The connection between the one-way objective and class frequency is examined both analytically and empirically. Appendix~\ref{app:theoretical_proof} studies a simplified scalar model in which samples from the same class share a common PEM output. Under quadratic output regularization, the unique optimum increases monotonically with the class count and has $\log N_c$ as its leading asymptotic term, together with a slower $\log\log N_c$ correction. Since $\log p(c)=\log N_c-\log N$, this establishes a direct relation between the learned signal and the log-prior term, while also making clear that exact prior recovery is not assumed. In the trained network, the estimates follow the overall class-frequency ordering on CIFAR-100-LT, with residual deviations that are associated with feature geometry.

We evaluate NPE-LA on long-tailed CIFAR-10 and CIFAR-100 under different imbalance ratios and optimization settings, comparing it with CE, cRT, and conventional logit adjustment. We also examine the effect of the number of PEMs on the shared backbone. To test whether the same construction extends beyond image-level classification, NPE-LA is further applied to semantic segmentation on STARE and ADE20K using frozen pretrained models and both linear and decoder-shaped PEMs. These experiments are intended to test the prior-estimation mechanism, its effect on minority classes, and its compatibility with different output structures.

The remainder of the paper is organized as follows. Section~\ref{sec:methodology} develops NPE and its integration with logit adjustment; Sections~\ref{sec:experiments}--\ref{sec:conclusion} present the experiments, results, and conclusions.

\section{Methodology}
\label{sec:methodology}
This section details the methodology of the Neural Prior Estimator (NPE) framework. First, the general formulation of the long-tailed learning problem is established. Next, the core estimation mechanism --- the Prior Estimation Module (PEM) --- and its optimization objective are presented. The imbalance-aware prediction mechanism, NPE-LA, is described in detail thereafter. Theoretical justification is provided in Appendix~\ref{app:theoretical_proof}.

The equations below are written for standard classification with vector-valued logits. The same construction can be applied to spatial outputs; Section~\ref{sec:experiments} evaluates that case in semantic segmentation.

\subsection{Problem Setup}

Let $\mathcal{D}=\{(\mathbf{x}_i,y_i)\}_{i=1}^{N}$ denote the training set, where $\mathbf{x}_i\in\mathcal{X}$ is an input sample and $y_i\in\{1,\ldots,C\}$ is its class label. 

For each class $c$, let $N_c$ denote the number of samples with label $c$, such that $\sum_{c=1}^{C} N_c = N$. The degree of imbalance is quantified by the imbalance ratio
\begin{equation}
	\rho = \frac{\max_{c} N_c}{\min_{c} N_c}.
\end{equation}
A large value of $\rho$ reflects a strongly non-uniform empirical class distribution, which is characteristic of imbalanced datasets.

The baseline approach for this classification task involves a standard classifier consisting of a feature extractor $f_\theta : \mathcal{X} \rightarrow \mathbb{R}^d$ followed by a linear prediction layer parameterized by $(\mathbf{W}, \mathbf{b})$:
\begin{equation}
	\mathbf{h} = f_\theta(\mathbf{x}), \qquad \mathbf{z} = \mathbf{W}\mathbf{h} + \mathbf{b},
\end{equation}
where $\mathbf{h} \in \mathbb{R}^d$ is the representation and $\mathbf{z} \in \mathbb{R}^C$ is the logit vector, with $z_c$ denoting the score associated with class $c$. The model is typically trained using the Cross-Entropy (CE) loss:
\begin{equation}
	\mathcal{L}_{\mathrm{CE}} 
	= \mathbb{E}_{(\mathbf{x},y) \sim \mathcal{D}}
	\bigg[
	- \log \frac{\exp(z_{y}(\mathbf{x}))}{\sum_{c=1}^{C} \exp(z_c(\mathbf{x}))}
	\bigg].
\end{equation}

Bayes' theorem reveals that minimizing the CE loss corresponds to maximizing the log-likelihood of the posterior $p_\theta(y|\mathbf{x})$, which decomposes as
\begin{equation}
	\log p_\theta(y|\mathbf{x}) 
	= \log p_\theta(\mathbf{x}|y) + \log p(y) + \mathrm{const}.
\end{equation}
Since neural classifiers implement decision rules in logit space, the empirical class prior $p(y)$ enters the decision function through its logarithm. Under severe imbalance, differences in $\log p(y)$ induce substantial implicit additive biases on the logits~\cite{menon2020long}, causing the model to systematically overestimate head classes and underestimate tail classes. 

Classical logit adjustment uses a fixed offset computed from empirical class counts. NPE replaces this externally supplied offset with a signal learned from the representation. The aim is not to define a new probabilistic prior conditioned on the input; rather, $\boldsymbol{\eta}(\mathbf{x})$ is used as a learned surrogate for the class-frequency component that appears in the logit decision rule. Under the idealized analysis in Appendix~\ref{app:theoretical_proof}, this signal becomes class dependent and its dominant term follows $\log N_c$. Away from that regime, the estimate can also vary with the representation, which is examined empirically in Appendix~\ref{app:empirical_result}.

\subsection{Neural Prior Estimator (NPE)}
\label{sec:NPE}

NPE uses one or more Prior Estimation Modules (PEMs) attached to the backbone representation and trained jointly with the main classifier. A single PEM is sufficient for the estimator; multiple PEMs are treated as an optional training extension.

\paragraph{Prior Estimation Modules (PEMs).}
Each PEM implements a differentiable mapping from the backbone feature vector to a 
class-wise output:
\begin{equation}
	\mathbf{u}_k(\mathbf{x})
	= g_k\!\left(\mathbf{h}(\mathbf{x})\right),
	\qquad k = 1,\ldots,N_{\mathrm{PEM}},
\end{equation}
where $\mathbf{u}_k(\mathbf{x})\in\mathbb{R}^C$ has the same dimensionality as the main classifier logits $\mathbf{z}(\mathbf{x})$. In the classification experiments, $g_k$ is a linear map. For a $d$-dimensional feature vector and $C$ classes, one linear PEM contains $dC+C$ parameters, the same order as a standard linear classification layer. Its parameter count is independent of the number of training samples and grows linearly with the number of classes and with $N_{\mathrm{PEM}}$.

\paragraph{Training objective.}
The loss for each PEM is defined as a one-way logistic term evaluated only on the true-class coordinate. The total NPE objective is the sum of the expected losses across all $N_{\text{PEM}}$ modules:
\begin{equation}
	\mathcal{L}_{\text{NPE}}
	=
	\sum_{k=1}^{N_{\text{PEM}}}
	\mathbb{E}_{(\mathbf{x}, y) \sim \mathcal{D}}
	[
	-\log \sigma\!\left( (-1)^{t}\, u_k(\mathbf{x})_{y} \right)
	],
\end{equation}
where $\sigma(u) = (1+e^{-u})^{-1}$ and $t \in \{0,1\}$ determines the sign convention of the update.

For the prediction task, the overall objective combines the CE loss of the main classifier with the NPE loss:
\begin{equation}
	\mathcal{L}
	=
	\mathcal{L}_{\text{CE}}
	+
	\mathcal{L}_{\text{NPE}}.
\end{equation}
This formulation allows the classifier and all PEMs to be optimized jointly, with the feature extractor receiving gradient signals from both objectives throughout training.
\paragraph{Emergent frequency-dependent structure.}

The one-way logistic loss acts only on the true-class coordinate and always pushes that coordinate in the same direction. Classes that occur more often therefore contribute more updates to their corresponding PEM coordinate. Regularization and saturation prevent unbounded growth, producing a frequency-dependent output scale. Appendix~\ref{app:theoretical_proof} isolates this effect in a scalar model and shows that the optimum is monotone in $N_c$ with an asymptotic $\log N_c$ leading term.

\paragraph{NPE estimate}
To consolidate the outputs of the  PEMs into a single classwise estimate, the \emph{NPE estimate} is defined as:
\begin{equation}
	\boldsymbol{\eta}(\mathbf{x})
	= \frac{(-1)^t}{N_{\mathrm{PEM}}} \sum_{k=1}^{N_{\mathrm{PEM}}} \mathbf{u}_k(\mathbf{x}).
\end{equation}
Here, $(-1)^t$ enforces a consistent sign convention independent of the training choice of $t$, and the normalization by $N_{\mathrm{PEM}}$ prevents the estimate from scaling directly with the number of PEMs.

This average is the operational NPE estimate. The theoretical analysis in Appendix~\ref{app:theoretical_proof} does not claim exact recovery of the empirical prior in a finite network. Instead, it shows that a simplified regularized PEM objective produces an optimum that is monotone in the class count and whose leading asymptotic term is logarithmic in that count.

\paragraph{Relation to the class prior.}
The empirical prior is $p_c=N_c/N$, so
\begin{equation}
    \log p_c = \log N_c - \log N.
\end{equation}
Thus, the leading $\log N_c$ term identified by the scalar analysis differs from $\log p_c$ only by a class-independent constant. The slower $\log\log N_c$ term derived in Appendix~\ref{app:theoretical_proof} introduces a class-dependent residual, so NPE should be interpreted as a learned log-prior surrogate rather than an exact estimator in the probabilistic sense.

\paragraph{Number of PEMs.}
A single PEM is sufficient to construct NPE. When $N_{\mathrm{PEM}}>1$, the estimate is averaged at inference, while the auxiliary losses are summed during training. The additional modules therefore change both the number of training-time parameters and the aggregate auxiliary gradient received by the backbone. The experiments report this behavior explicitly rather than treating the number of PEMs as part of the definition of NPE.

\paragraph{Choice of sign convention.}
The parameter $t$ determines the direction of the update applied to the true-class coordinate. 
The NPE estimate is re-signed so that its interpretation is consistent across the two sign conventions. 
Empirically, the subtractive configuration ($t=1$) performs favorably and is adopted in all subsequent experiments.

\subsection{NPE for Imbalance-Aware Prediction}
\label{sec:npe_la}

Since $\log N_c$ and $\log p_c$ differ only by a class-independent constant, the NPE estimate $\boldsymbol{\eta}(\mathbf{x})$ can be used as a learned surrogate for the log prior, up to a class-independent shift. Accordingly, NPE-LA applies $\boldsymbol{\eta}(\mathbf{x})$ as a feature-conditioned prior correction:
\begin{equation}
	\tilde{\mathbf{z}}(\mathbf{x})
	= \mathbf{z}(\mathbf{x}) - \boldsymbol{\eta}(\mathbf{x}).
\end{equation}

The procedure is structurally related to classical logit adjustment~\cite{menon2020long},
but differs in two key aspects.
First, the effective prior is \emph{learned jointly} during training rather than fixed in advance.
Second, the correction is \emph{feature-dependent}: the adjustment responds to the local behavior of the representation $\mathbf{h}(\mathbf{x})$.

\paragraph{Inference-time efficiency.}
For linear PEMs, let $g_k(\mathbf{h})=\mathbf{A}_k\mathbf{h}+\mathbf{a}_k$. The adjusted logits can be written as a single linear layer,
\begin{equation}
\tilde{\mathbf{z}}
=
\left(\mathbf{W}-\frac{(-1)^t}{N_{\mathrm{PEM}}}\sum_k\mathbf{A}_k\right)\mathbf{h}
+
\left(\mathbf{b}-\frac{(-1)^t}{N_{\mathrm{PEM}}}\sum_k\mathbf{a}_k\right).
\end{equation}
The PEMs can therefore be folded into the classifier after training, so the linear version adds no inference-time layer or forward-pass cost.

\section{Experiments}
\label{sec:experiments}
This section describes the experimental setup used to evaluate NPE-LA across different tasks. Both image-level classification on CIFAR datasets and pixel-level semantic segmentation on STARE and ADE20K are considered. The experiments are designed to assess performance under imbalanced class distributions.

All PEMs are initialized with a normal distribution having a small standard deviation of $0.001$, ensuring that the initial NPE estimate is close to zero. This choice prevents the one-way logistic loss from operating in a saturated regime at the start of training, preserving informative gradients and avoiding early optimization collapse.

\subsection{Classification Experiments}

NPE-LA is evaluated on long-tailed CIFAR-10 and CIFAR-100~\cite{krizhevsky2009learning} using a ResNet-32 backbone~\cite{he2016deep} trained from scratch with SGD (momentum $0.9$) and standard data augmentation. Classes are grouped by sample count: Head classes contain more than 100 samples, Medium classes contain between 20 and 100 samples, and Tail classes contain fewer than 20 samples.

Robustness to training conditions is assessed using two hyperparameter configurations summarized in \autoref{tab:hparams}. HP-1 follows the settings introduced in~\cite{cao2019learning}, which are widely used in imbalance-aware prediction studies, while HP-2 is a custom configuration designed to examine sensitivity to different optimization dynamics. NPE-LA is compared against CE, cRT, and LA under identical conditions. All reported results are mean~$\pm$~std across three independent runs, and final-epoch accuracy is used as the evaluation metric.

\begin{table}[!tbh]
	\centering
	\caption{Training hyperparameters for classification experiments. LR decay factor = 0.1 at each milestone.}
	\label{tab:hparams}
	\begin{tabular}{lccccc}
		\toprule
		Setting & LR & Weight Decay & Batch Size & Epochs & LR decay milestones \\ 
		\midrule
		HP-1 & 0.1 & $2\times10^{-4}$ & 124 & 200 & 160, 180 \\
		HP-2 & 0.05 & $10^{-3}$ & 64 & 120 & 100, 110 \\
		\bottomrule
	\end{tabular}
\end{table}

\subsection{Semantic Segmentation Experiments}

NPE-LA is evaluated on semantic segmentation using two imbalanced datasets: STARE~\cite{hoover2000locating} and ADE20K~\cite{zhou2017scene}. STARE contains 20 retinal images with two classes, where vessel pixels comprise only about 14\% of the image, creating a pronounced pixel-level imbalance. ADE20K includes 20{,}000 images and 150 categories with substantial frequency skew, with some classes appearing infrequently.

Pretrained backbones and decoder heads from MMSegmentation~\cite{mmseg2020} are used and kept frozen during PEM training. Freezing the feature extractor isolates the effect of the PEM and ensures that performance changes stem solely from logit recalibration rather than altered representations. For STARE, a UNet~\cite{ronneberger2015u} with an FCN head~\cite{long2015fully} is used, while ADE20K experiments employ either DeepLabv3~\cite{chen2017rethinking} or Swin-T~\cite{liu2021Swin} paired with a UPerNet decoder~\cite{xiao2018unified}. Optimizer hyperparameters, including learning rate, momentum, and weight decay, follow MMSegmentation defaults, and the total training iterations and learning-rate decay schedule follow the custom configuration in \autoref{tab:seg_schedule}.

\begin{table}[!tbh]
	\centering
	\caption{Per-model training schedule and optimizer settings. LR decay factor = 0.1 at each milestone.}
	\label{tab:seg_schedule}
	\renewcommand{\arraystretch}{1.15}
	\begin{tabular}{lcccccc}
		\toprule
		Dataset & Backbone + Head & Optimizer & LR & WD & Iterations & LR decay milestones \\
		\midrule
		STARE   & UNet-FCN        & SGD   & 1e-2 & 5e-4 & 600   & 200, 400            \\
		ADE20K  & DeepLab-V3      & SGD   & 1e-2 & 5e-4 & 2\,500 & 1\,250              \\
		ADE20K  & Swin-T + UPerNet & AdamW & 6e-5 & 1e-2 & 2\,500 & 1\,250              \\
		\bottomrule
	\end{tabular}
\end{table}

For dense prediction we consider two PEM variants. The first is a single-layer FCN, giving a linear feature-to-estimate mapping. The second follows the structure of the corresponding decoder head and therefore includes non-linear transformations. This comparison tests whether the correction depends strongly on PEM capacity.

\paragraph{PEM normalization and scaling.}
Batch Normalization is not applied to the final PEM output because its channel scale carries the frequency-dependent signal used for correction. In the dense-prediction experiments, the frozen decoder logits and the PEM outputs are produced by different objectives and can have substantially different numerical scales. We therefore use a scalar factor $\alpha$,
\[
\boldsymbol{\eta}_{\mathrm{scaled}}(\mathbf{x})=\alpha\,\boldsymbol{\eta}(\mathbf{x}),
\]
before applying the correction. The factor controls the magnitude of recalibration; it is an experimental hyperparameter in the dense-prediction setting and is reported in Tables~\ref{tab:stare_unet_results} and~\ref{tab:ade20k_results}.

\section{Results and Discussion} 
\label{sec:results} 
This section presents an empirical evaluation of the Neural Prior Estimator framework. The analysis begins by examining the regularization effect induced by PEMs on the backbone, then evaluates the full $\mathbf{NPE{+}LA}$ mechanism on long-tailed classification under varying optimization regimes, and finally considers its application to semantic segmentation.
 
Appendix~\ref{app:empirical_result} compares the learned NPE signal with the empirical class prior on CIFAR-100-LT at $\rho=100$. The estimate follows the overall frequency ordering, but the residual is not uniform across classes. In particular, classes with less compact representations tend to be overestimated. This observation is useful because it marks a limit of interpreting $\boldsymbol{\eta}(\mathbf{x})$ as an exact prior: the estimate also reflects representation geometry. For NPE-LA, this can change the correction applied to classes with dispersed features and may affect their margins~\cite{bartlett2017spectrally,cao2019learning}.

\subsection{Classification Results}

Before evaluating the full NPE-LA framework, we isolate the implicit impact of the NPE training objective on the shared feature backbone. In this ablation, the PEMs are active only during training; at inference time, the standard classifier is used without any logit adjustment.

Models were trained deterministically with varying numbers of PEMs ($N_{\mathrm{PEM}} \in \{0, 1, 4, 8, 16\}$). Figure~\ref{fig:ach_classwise} illustrates the resulting mean accuracy across Head, Medium, and Tail subgroups on CIFAR-100 ($\rho = 100$).

\begin{figure}[!tbh]
	\begin{center}
		\includegraphics[width=0.7\linewidth]{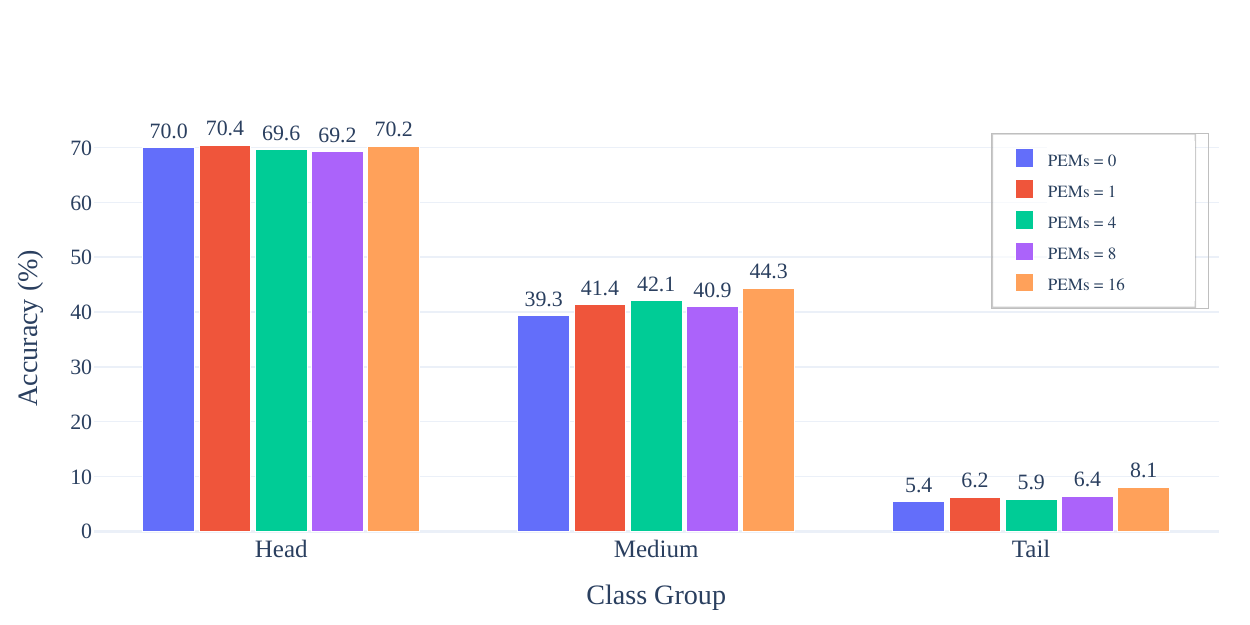}
		\caption{Class-wise accuracy on CIFAR-100 ($\rho = 100$) under HP-2, illustrating the impact of different numbers of PEMs ($N_{\mathrm{PEM}}$) across Head, Medium, and Tail classes. NPE is utilized during training only; no logit adjustment is performed at inference time.}
		\label{fig:ach_classwise}
	\end{center}
\end{figure}

The training-only PEM ablation changes the class-wise balance even when the NPE correction is not used at inference. Head accuracy remains close to the baseline, while medium and tail accuracy generally increases for the larger PEM configurations. The trend is not monotonic for every value of $N_{\mathrm{PEM}}$, which suggests that the effect is tied to the optimization dynamics rather than to a simple ensemble benefit.

\autoref{tab:pem_effect_hp2} reports the mean accuracy on CIFAR-100 and CIFAR-10 for the same experimental setting.  
The aggregate results in Table~\ref{tab:pem_effect_hp2} are mixed for small numbers of PEMs but show gains in several settings for the larger configurations. Because the PEM losses are summed, this ablation changes both model capacity and the strength of the auxiliary training signal; the table should therefore be read as an empirical study of the current implementation rather than as evidence that more PEMs are always better.

\begin{table}[!tbh]
	\centering
	\caption{Top-1 accuracy (\%) on CIFAR-100 and CIFAR-10 for different imbalance ratios $\rho$ and PEMs under hyperparameter setting HP-2. NPE is utilized during training only; no logit adjustment is performed at inference time.}
	\label{tab:pem_effect_hp2}
	\begin{tabular}{lccc|ccc}
		\toprule
		\multirow{2}{*}{\# of PEMs} & \multicolumn{3}{c|}{CIFAR-100} & \multicolumn{3}{c}{CIFAR-10} \\
		\cmidrule(lr){2-4} \cmidrule(lr){5-7}
		& 200 & 100 & 50 & 200 & 100 & 50 \\
		\midrule
		0 & 37.62 & 40.53 & 45.98 & 64.40  & \textbf{75.01}  & 77.00 \\
		1 & 37.28 & 41.00 & 45.96 & 65.24 & 73.05 & 78.29 \\
		4 & 37.09 & \textbf{42.35} & 45.45 & \textbf{66.31} & 73.95 & 77.99 \\
		8 & 38.39 & 41.65 & 46.29 & 66.13 & 73.72 & \textbf{78.98} \\
		16 & \textbf{38.42} & 41.78 & \textbf{46.36} & 66.22 & 74.20 & 78.33 \\
		\bottomrule
	\end{tabular}
	
\end{table}

Following the analysis of PEM-only training, the next experiments evaluate the full NPE-LA mechanism, where the learned prior offsets are applied during inference. Classification performance on CIFAR-100 and CIFAR-10 under hyperparameter settings HP-1 and HP-2 is reported in \autoref{tab:classification_results_hp1} and \autoref{tab:classification_results_hp2}, respectively.

Across the two training settings, NPE-LA remains close to LA while providing clear gains over CE and cRT in most cases. Under HP-1, the 16-PEM configuration has the highest mean accuracy in five of the six dataset/imbalance settings in Table~\ref{tab:classification_results_hp1}; LA is best on CIFAR-100 at $\rho=50$. Under HP-2, the 16-PEM configuration is best in all three CIFAR-100 settings and on CIFAR-10 at $\rho=50$, while LA is slightly higher on CIFAR-10 at $\rho=200$ and $\rho=100$.

The single-PEM variant is also informative because it represents the minimal NPE construction. Its results are generally close to LA, although it does not consistently exceed it. The larger gains with 16 PEMs should be interpreted together with the training-only ablation above: increasing $N_{\mathrm{PEM}}$ changes the aggregate auxiliary gradient as well as the number of added heads.

\begin{table}[!tbh]
	\centering
	\caption{Top-1 accuracy (\%) on CIFAR-100 and CIFAR-10 for different imbalance ratios $\rho$ and methods under hyperparameter setting HP-1. Bold indicates the highest value; underlined indicates the second highest}
	\label{tab:classification_results_hp1}
	\resizebox{\textwidth}{!}{
		\begin{tabular}{lccc|ccc}
			\toprule
			\multirow{2}{*}{Method} & \multicolumn{3}{c|}{CIFAR-100} & \multicolumn{3}{c}{CIFAR-10} \\
			\cmidrule(lr){2-4} \cmidrule(lr){5-7}
			& 200 & 100 & 50 & 200 & 100 & 50 \\
			\midrule
			CE & 34.83 $\pm$ 0.26 & 39.14 $\pm$ 0.41 & 43.87 $\pm$ 0.25 & 63.54 $\pm$ 0.28 & 71.24 $\pm$ 0.29 & 76.77 $\pm$ 0.37 \\
			cRT & 37.04 $\pm$ 0.24 & 41.04 $\pm$ 0.14 & 45.83 $\pm$ 0.32 & 67.19 $\pm$ 0.31 & 73.80 $\pm$ 0.42 & 79.07 $\pm$ 0.51 \\
			LA & 38.59 $\pm$ 0.18 & \underline{42.63} $\pm$ 0.46 & \textbf{47.07} $\pm$ 0.49 & 72.10 $\pm$ 0.27 & 77.37 $\pm$ 0.32 & 81.25 $\pm$ 0.38 \\
			NPE-LA (1 PEM) & \underline{38.85} $\pm$ 0.30 & 42.53 $\pm$ 0.40 & \underline{46.86} $\pm$ 0.78 & \underline{72.42} $\pm$ 0.21 & \underline{77.89} $\pm$ 0.38 & \underline{82.22} $\pm$ 0.46 \\
			NPE-LA (16 PEMs) & \textbf{39.36} $\pm$ 0.12 & \textbf{42.73} $\pm$ 0.28 & 46.44 $\pm$ 0.24 & \textbf{73.18} $\pm$ 0.41 & \textbf{78.17} $\pm$ 0.73 & \textbf{82.31} $\pm$ 0.22 \\
			\bottomrule
		\end{tabular}
	}
	
\end{table}

\begin{table}[!tbh]
	\centering
	\caption{Top-1 accuracy (\%) on CIFAR-100 and CIFAR-10 for different imbalance ratios $\rho$ and methods under hyperparameter setting HP-2. Bold indicates the highest value; underlined indicates the second highest}
	\label{tab:classification_results_hp2}
	\resizebox{\textwidth}{!}{
		\begin{tabular}{lccc|ccc}
			\toprule
			\multirow{2}{*}{Method} & \multicolumn{3}{c|}{CIFAR-100} & \multicolumn{3}{c}{CIFAR-10} \\
			\cmidrule(lr){2-4} \cmidrule(lr){5-7}
			& 200 & 100 & 50 & 200 & 100 & 50 \\
			\midrule
			CE & 37.04 $\pm$ 0.23 & 41.00 $\pm$ 0.12 & 45.46 $\pm$ 0.11 & 65.59 $\pm$ 0.20 & 73.18 $\pm$ 0.77 & 77.76 $\pm$ 0.44 \\
			cRT & 41.58 $\pm$ 0.24 & 45.43 $\pm$ 0.11 & 50.50 $\pm$ 0.12 & 72.34 $\pm$ 0.71 & 77.85 $\pm$ 0.25 & 81.01 $\pm$ 0.50 \\
			LA & \underline{43.35} $\pm$ 0.33 & \underline{46.85} $\pm$ 0.53 & \underline{51.30} $\pm$ 0.22 & \textbf{74.91} $\pm$ 0.51 & \textbf{80.13} $\pm$ 0.13 & \underline{82.95} $\pm$ 0.36 \\
			NPE-LA (1 PEM) & 42.99 $\pm$ 0.30 & 46.72 $\pm$ 0.23 & 51.25 $\pm$ 0.36 & 74.37 $\pm$ 0.25 & 79.35 $\pm$ 0.37 & 82.32 $\pm$ 0.19 \\
			NPE-LA (16 PEMs) & \textbf{43.96} $\pm$ 0.40 & \textbf{47.47} $\pm$ 0.36 &\textbf{ 51.82} $\pm$ 0.17 & \underline{74.83} $\pm$ 0.55 & \underline{80.12} $\pm$ 0.31 & \textbf{83.16} $\pm$ 0.08 \\
			\bottomrule
		\end{tabular}
	}
	
\end{table}

\autoref{fig:method_classwise} shows class-wise accuracy on CIFAR-100 for the high-imbalance setting ($\rho = 100$) under HP-2. CE prioritizes head classes, leaving tail categories largely unrecognized. cRT shifts performance toward minority classes, improving medium and tail accuracy at the expense of head performance. LA pushes this redistribution further, providing the strongest tail gains among the baselines while imposing the largest reduction on head classes. NPE-LA moves in the same direction but achieves a more favorable balance. NPE-LA with even a single PEM outperforms cRT on medium and tail classes, albeit with lower performance on the head classes. Employing sixteen PEMs brings tail-class accuracy close to that of LA while achieving a more balanced overall performance.

\begin{figure}[!tbh]
	\begin{center}
		\includegraphics[width=0.7\linewidth]{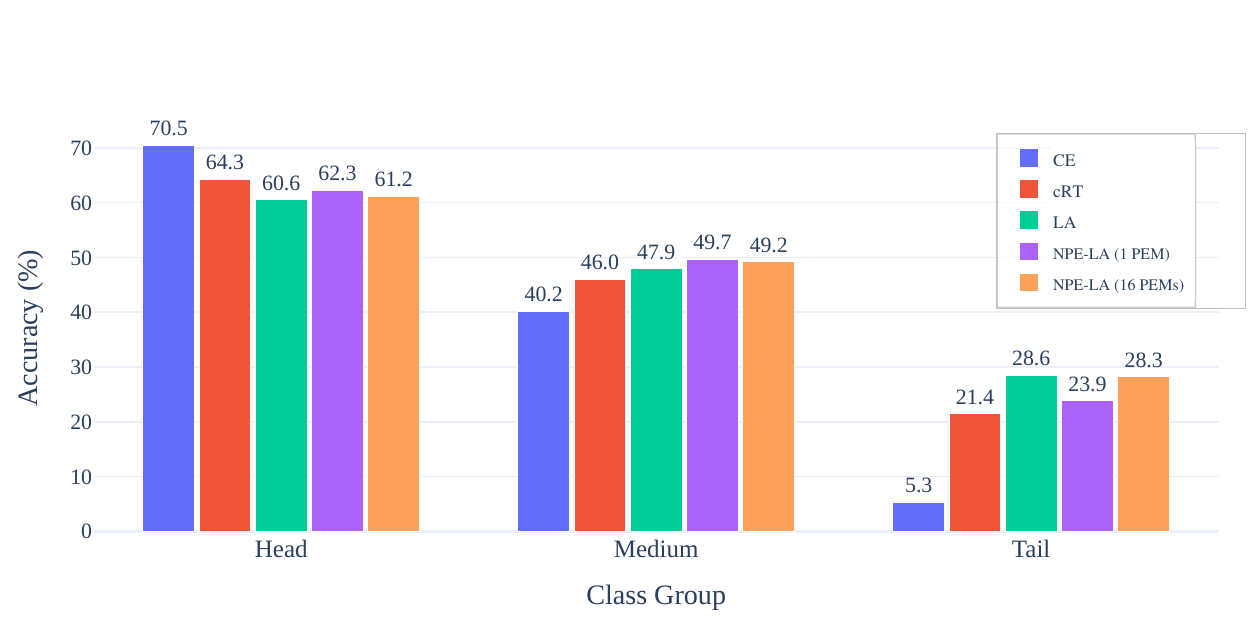}
		\caption{Class-wise accuracy on CIFAR-100 ($\rho=100$) under HP-2, comparing different methods across head, medium, and tail classes.}
		\label{fig:method_classwise}
	\end{center}
\end{figure}

\subsection{Segmentation Results}
In contrast to the classification experiments, the segmentation experiments are conducted with a frozen backbone. This constraint implies that the NPE does not influence the shared backbone features, leading to the evaluation of NPE-LA exclusively in a single PEM configuration.

The segmentation performance on the STARE dataset with a UNet backbone under different scale factors is summarized in \autoref{tab:stare_unet_results}. With $\alpha=1$, NPE-LA strongly shifts predictions toward the foreground: foreground accuracy rises, but mean Dice falls because of overprediction. Reducing the scale to $\alpha=0.2$ produces a better balance and improves both mean Dice and mean accuracy over the baseline. The result shows that the learned correction transfers to pixel-level prediction, while also making clear that its scale must be controlled in this setting.

\begin{table}[!tbh]
	\centering
	\caption{STARE segmentation with UNet--FCN: Dice and pixel accuracy for background (BG), foreground (FG), and mean, comparing the baseline and NPE-LA across scale factors $\alpha$.}
	\label{tab:stare_unet_results}
		\begin{tabular}{l l l c c c c c c c}
			\toprule
			Backbone & Decode Head & PEM & Scale & \multicolumn{3}{c}{Dice (\%)} & \multicolumn{3}{c}{Accuracy (\%)} \\
			\cmidrule(lr){5-7} \cmidrule(lr){8-10}
			&  & Type & Factor & BG & FG & Mean & BG & FG & Mean \\
			\midrule
			\multirow{3}{*}{UNet} & \multirow{3}{*}{FCN} & - & - & \textbf{98.55} & 81.20 & 89.78 & \textbf{99.13} & 75.21 & 87.17 \\
			\cmidrule(lr){3-10}
			& & \multirow{2}{*}{FCN} & 1 & 97.34 & 75.02 & 86.27 & 95.22 & \textbf{94.93} & \textbf{95.08} \\
			&  &  & 0.2 & 98.52 & \textbf{82.39} & \textbf{90.46} & 98.44 & 83.21 & 90.83 \\
			\bottomrule
		\end{tabular}%
	
\end{table}

ADE20K shows the same sensitivity to the correction scale (Table~\ref{tab:ade20k_results}). For DeepLab-V3, the unscaled correction increases mAcc but substantially reduces mIoU. With $\alpha=0.1$, mIoU returns to approximately the baseline level while mAcc remains higher. The FCN PEM gives a similar result despite using only a single linear projection. For Swin-T with UPerNet, $\alpha=1$ is unstable, whereas $\alpha=0.1$ restores performance; the FCN PEM at $\alpha=0.1$ slightly exceeds the baseline mIoU and improves mAcc. These experiments support the portability of the correction across decoder architectures, but the gains are modest and depend on choosing an appropriate scale.

\begin{table}[!tbh]
	\centering
	\caption{ADE20K segmentation with DeepLabv3 and Swin-T: mIoU and pixel accuracy for baseline and NPE-LA with different PEM types and scale factors $\alpha$.}
	\label{tab:ade20k_results}
		\begin{tabular}{l l l c c c}
			\toprule
			Backbone & Decode Head & PEM Type & Scale Factor & mIoU (\%) & mAcc (\%) \\
			\midrule
			\multirow{5}{*}{DeepLab-V3} & \multirow{5}{*}{ASPPHead} & - & - & 42.42 & 53.60 \\
			\cmidrule(lr){3-6}
			&  & \multirow{2}{*}{ASPPHead} & 1 & 24.55 & 65.59  \\
			&  &  & 0.1 & 42.82 & 57.99 \\
			\cmidrule(lr){3-6}
			&  & \multirow{2}{*}{FCN} & 1 & 29.28 & 67.05  \\
			&  &  & 0.1 & 42.92 & 56.92 \\	
			\midrule
			\midrule
			\multirow{5}{*}{Swin-T} & \multirow{5}{*}{UPerNet} & - & - & 44.41 & 55.96 \\
			\cmidrule(lr){3-6}
			& &\multirow{2}{*}{UPerNet} & 1 & 12.29 & 32.51  \\
			&  &  & 0.1 & 42.15 & 62.73 \\
			\cmidrule(lr){3-6}
			& & \multirow{2}{*}{FCN} & 1 & 33.87 & 66.51  \\
			&  &  & 0.1 & 44.73 & 58.00 \\
			\bottomrule
		\end{tabular}%
	
\end{table}

\autoref{fig:class_iou_ade20k} and \autoref{fig:class_acc_ade20k} show class-wise IoU and accuracy for baseline DeepLab-V3 and the NPE-LA variant with an FCN PEM and scale factor of 0.1. IoU decreases slightly for classes that already perform well but increases for weaker classes. Class-wise accuracy generally improves, reflecting NPE-LA's effect of redistributing confidence without destabilizing overall predictions.

\begin{figure}[!tp]
	\centering
	\includegraphics[width=0.88\linewidth]{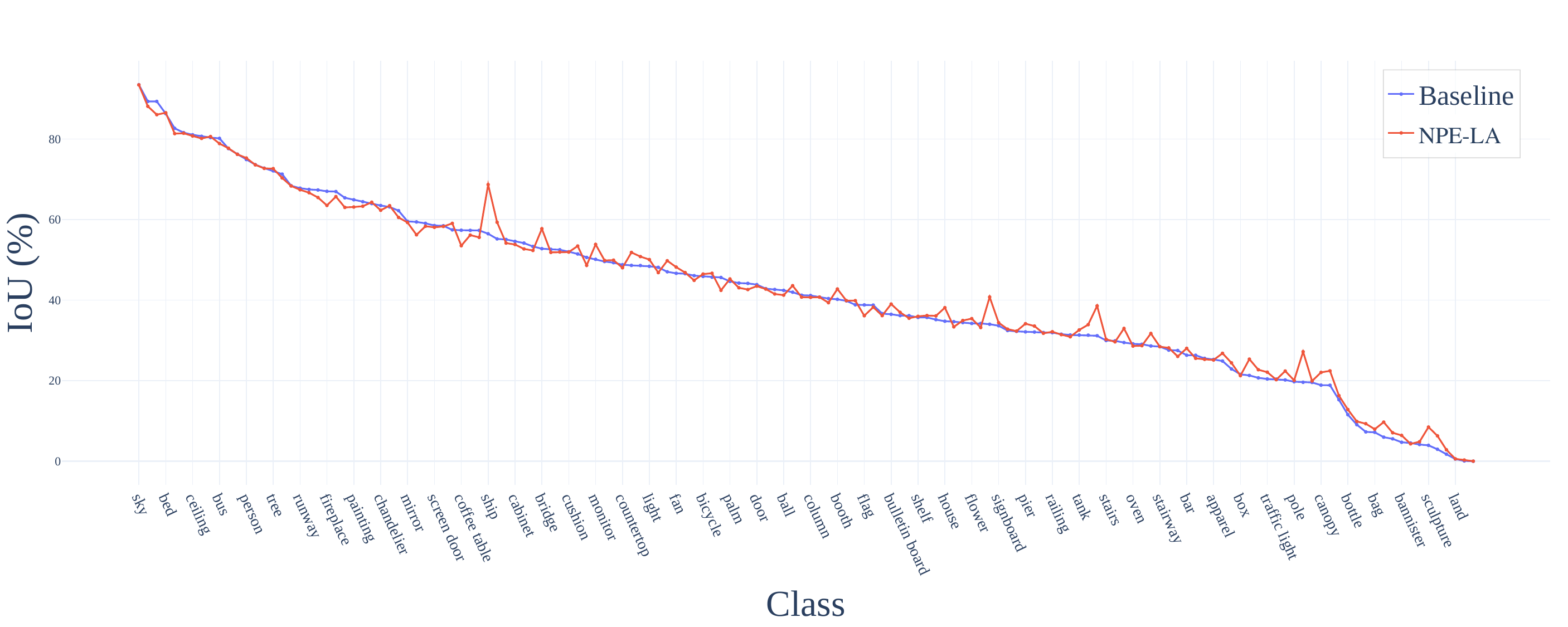}
		\caption{Class-wise IoU on ADE20K for baseline DeepLab-V3 and NPE-LA variant with a single FCN PEM and scale factor 0.1. Baseline IoUs are ordered in ascending order for improved visual clarity.}
		\label{fig:class_iou_ade20k}
\end{figure}

\begin{figure}[!tp]
	\centering
	\includegraphics[width=0.88\linewidth]{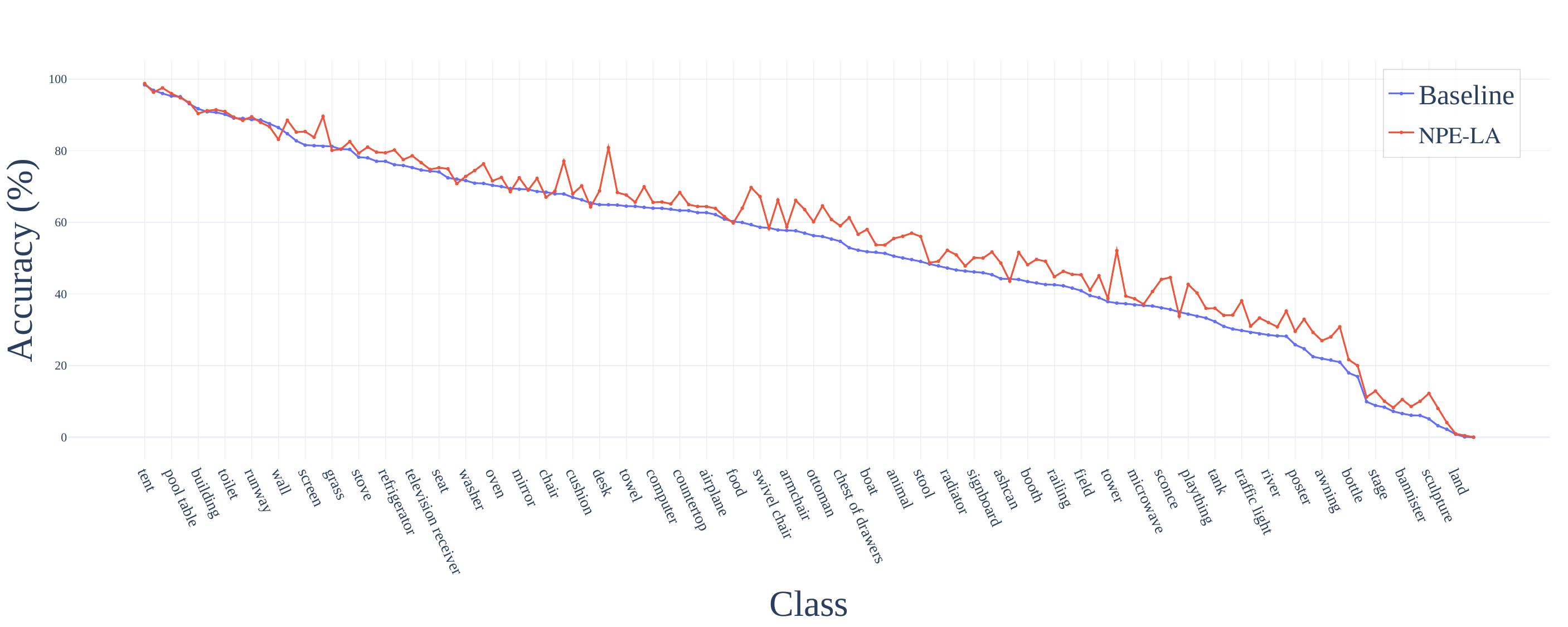}
		\caption{Class-wise accuracy on ADE20K for baseline DeepLab-V3 and NPE-LA variant using a single FCN PEM and scale factor 0.1. Baseline accuracies are ordered in ascending order for improved visual clarity.}
		\label{fig:class_acc_ade20k}
\end{figure}

\FloatBarrier

Overall, the STARE and ADE20K experiments suggest that NPE-LA can adjust predictions for underrepresented classes while maintaining stable performance for the main decoder, even with simple PEM configurations and frozen backbones.

\FloatBarrier
\section{Conclusion}
\label{sec:conclusion}

We introduced Neural Prior Estimation as a way to learn the class-frequency signal used by logit adjustment from latent representations rather than supplying empirical class counts directly to the correction rule. A PEM is trained with a one-way logistic objective, and its output is used as a class-wise logit offset. In the simplified analysis, the regularized scalar optimum is monotone in class frequency and has $\log N_c$ as its leading asymptotic term. The empirical analysis shows the same overall frequency ordering, while also revealing deviations associated with representation geometry.

On long-tailed CIFAR, NPE-LA is competitive with standard logit adjustment and improves over CE and cRT in the reported settings. A single PEM is sufficient for the method, while additional PEMs can alter the backbone optimization and sometimes improve the final result. The dense-prediction experiments on STARE and ADE20K show that the learned correction can also be applied to pixel-level outputs, although a scale factor is needed to match the correction to the frozen decoder logits.

The current results do not settle how NPE behaves at ImageNet-LT or iNaturalist scale, nor do they separate all effects introduced by multiple PEMs. Those experiments are the most important next step. More broadly, the observed interaction between the estimated frequency signal and feature geometry suggests that prior estimation from learned representations is not simply a replacement for counting; it may also provide a way to study how imbalance is encoded by the representation itself.

\bibliographystyle{unsrt}
\bibliography{references.bib}  

@article{zhang2023deep,
  title={Deep long-tailed learning: A survey},
  author={Zhang, Yifan and Kang, Bingyi and Hooi, Bryan and Yan, Shuicheng and Feng, Jiashi},
  journal={IEEE transactions on pattern analysis and machine intelligence},
  volume={45},
  number={9},
  pages={10795--10816},
  year={2023},
  publisher={IEEE}
}

@article{wang2020long,
  title={Long-tailed recognition by routing diverse distribution-aware experts},
  author={Wang, Xudong and Lian, Long and Miao, Zhongqi and Liu, Ziwei and Yu, Stella X},
  journal={arXiv preprint arXiv:2010.01809},
  year={2020}
}

@inproceedings{kim2020m2m,
  title={M2m: Imbalanced classification via major-to-minor translation},
  author={Kim, Jaehyung and Jeong, Jongheon and Shin, Jinwoo},
  booktitle={Proceedings of the IEEE/CVF conference on computer vision and pattern recognition},
  pages={13896--13905},
  year={2020}
}

@inproceedings{wang2021rsg,
  title={Rsg: A simple but effective module for learning imbalanced datasets},
  author={Wang, Jianfeng and Lukasiewicz, Thomas and Hu, Xiaolin and Cai, Jianfei and Xu, Zhenghua},
  booktitle={Proceedings of the IEEE/CVF conference on computer vision and pattern recognition},
  pages={3784--3793},
  year={2021}
}

@article{zhang2025systematic,
  title={A systematic review on long-tailed learning},
  author={Zhang, Chongsheng and Almpanidis, George and Fan, Gaojuan and Deng, Binquan and Zhang, Yanbo and Liu, Ji and Kamel, Aouaidjia and Soda, Paolo and Gama, Jo{\~a}o},
  journal={IEEE Transactions on Neural Networks and Learning Systems},
  year={2025},
  publisher={IEEE}
}

@article{yin2018feature,
  title={Feature transfer learning for deep face recognition with long-tail data},
  author={Yin, Xi and Yu, Xiang and Sohn, Kihyuk and Liu, Xiaoming and Chandraker, Manmohan},
  journal={arXiv preprint arXiv:1803.09014},
  volume={1},
  number={2},
  year={2018},
  publisher={CoRR}
}

@inproceedings{li2021metasaug,
  title={Metasaug: Meta semantic augmentation for long-tailed visual recognition},
  author={Li, Shuang and Gong, Kaixiong and Liu, Chi Harold and Wang, Yulin and Qiao, Feng and Cheng, Xinjing},
  booktitle={Proceedings of the IEEE/CVF conference on computer vision and pattern recognition},
  pages={5212--5221},
  year={2021}
}

@inproceedings{wang2019dynamic,
  title={Dynamic curriculum learning for imbalanced data classification},
  author={Wang, Yiru and Gan, Weihao and Yang, Jie and Wu, Wei and Yan, Junjie},
  booktitle={Proceedings of the IEEE/CVF international conference on computer vision},
  pages={5017--5026},
  year={2019}
}

@article{chawla2002smote,
  title={SMOTE: synthetic minority over-sampling technique},
  author={Chawla, Nitesh V and Bowyer, Kevin W and Hall, Lawrence O and Kegelmeyer, W Philip},
  journal={Journal of artificial intelligence research},
  volume={16},
  pages={321--357},
  year={2002}
}

@inproceedings{zang2021fasa,
  title={Fasa: Feature augmentation and sampling adaptation for long-tailed instance segmentation},
  author={Zang, Yuhang and Huang, Chen and Loy, Chen Change},
  booktitle={Proceedings of the IEEE/CVF international conference on computer vision},
  pages={3457--3466},
  year={2021}
}

@inproceedings{cui2019class,
  title={Class-balanced loss based on effective number of samples},
  author={Cui, Yin and Jia, Menglin and Lin, Tsung-Yi and Song, Yang and Belongie, Serge},
  booktitle={Proceedings of the IEEE/CVF conference on computer vision and pattern recognition},
  pages={9268--9277},
  year={2019}
}

@inproceedings{lin2017focal,
  title={Focal loss for dense object detection},
  author={Lin, Tsung-Yi and Goyal, Priya and Girshick, Ross and He, Kaiming and Doll{\'a}r, Piotr},
  booktitle={Proceedings of the IEEE international conference on computer vision},
  pages={2980--2988},
  year={2017}
}

@article{ren2020balanced,
  title={Balanced meta-softmax for long-tailed visual recognition},
  author={Ren, Jiawei and Yu, Cunjun and Ma, Xiao and Zhao, Haiyu and Yi, Shuai and others},
  journal={Advances in neural information processing systems},
  volume={33},
  pages={4175--4186},
  year={2020}
}

@inproceedings{tan2020equalization,
  title={Equalization loss for long-tailed object recognition},
  author={Tan, Jingru and Wang, Changbao and Li, Buyu and Li, Quanquan and Ouyang, Wanli and Yin, Changqing and Yan, Junjie},
  booktitle={Proceedings of the IEEE/CVF conference on computer vision and pattern recognition},
  pages={11662--11671},
  year={2020}
}

@inproceedings{hsieh2021droploss,
  title={Droploss for long-tail instance segmentation},
  author={Hsieh, Ting-I and Robb, Esther and Chen, Hwann-Tzong and Huang, Jia-Bin},
  booktitle={Proceedings of the AAAI conference on artificial intelligence},
  volume={35},
  number={2},
  pages={1549--1557},
  year={2021}
}

@inproceedings{wang2021adaptive,
  title={Adaptive class suppression loss for long-tail object detection},
  author={Wang, Tong and Zhu, Yousong and Zhao, Chaoyang and Zeng, Wei and Wang, Jinqiao and Tang, Ming},
  booktitle={Proceedings of the IEEE/CVF conference on computer vision and pattern recognition},
  pages={3103--3112},
  year={2021}
}

@article{cao2019learning,
  title={Learning imbalanced datasets with label-distribution-aware margin loss},
  author={Cao, Kaidi and Wei, Colin and Gaidon, Adrien and Arechiga, Nikos and Ma, Tengyu},
  journal={Advances in neural information processing systems},
  volume={32},
  year={2019}
}

@article{menon2020long,
  title={Long-tail learning via logit adjustment},
  author={Menon, Aditya Krishna and Jayasumana, Sadeep and Rawat, Ankit Singh and Jain, Himanshu and Veit, Andreas and Kumar, Sanjiv},
  journal={arXiv preprint arXiv:2007.07314},
  year={2020}
}

@article{tian2020posterior,
  title={Posterior re-calibration for imbalanced datasets},
  author={Tian, Junjiao and Liu, Yen-Cheng and Glaser, Nathaniel and Hsu, Yen-Chang and Kira, Zsolt},
  journal={Advances in neural information processing systems},
  volume={33},
  pages={8101--8113},
  year={2020}
}

@inproceedings{zhang2021distribution,
  title={Distribution alignment: A unified framework for long-tail visual recognition},
  author={Zhang, Songyang and Li, Zeming and Yan, Shipeng and He, Xuming and Sun, Jian},
  booktitle={Proceedings of the IEEE/CVF conference on computer vision and pattern recognition},
  pages={2361--2370},
  year={2021}
}

@misc{krizhevsky2009learning,
  title={Learning multiple layers of features from tiny images},
  author={Krizhevsky, Alex and Hinton, Geoffrey and others},
  year={2009},
  publisher={Toronto, ON, Canada}
}

@inproceedings{he2016deep,
  title={Deep residual learning for image recognition},
  author={He, Kaiming and Zhang, Xiangyu and Ren, Shaoqing and Sun, Jian},
  booktitle={Proceedings of the IEEE conference on computer vision and pattern recognition},
  pages={770--778},
  year={2016}
}

@misc{mmseg2020,
    title={{MMSegmentation}: OpenMMLab Semantic Segmentation Toolbox and Benchmark},
    author={MMSegmentation Contributors},
    howpublished = {\url{https://github.com/open-mmlab/mmsegmentation}},
    year={2020}
}

@article{kang2019decoupling,
  title={Decoupling representation and classifier for long-tailed recognition},
  author={Kang, Bingyi and Xie, Saining and Rohrbach, Marcus and Yan, Zhicheng and Gordo, Albert and Feng, Jiashi and Kalantidis, Yannis},
  journal={arXiv preprint arXiv:1910.09217},
  year={2019}
}

@article{ye2020identifying,
  title={Identifying and compensating for feature deviation in imbalanced deep learning},
  author={Ye, Han-Jia and Chen, Hong-You and Zhan, De-Chuan and Chao, Wei-Lun},
  journal={arXiv preprint arXiv:2001.01385},
  year={2020}
}

@inproceedings{wu2020solving,
  title={Solving long-tailed recognition with deep realistic taxonomic classifier},
  author={Wu, Tz-Ying and Morgado, Pedro and Wang, Pei and Ho, Chih-Hui and Vasconcelos, Nuno},
  booktitle={European Conference on Computer Vision},
  pages={171--189},
  year={2020},
  organization={Springer}
}

@article{tang2020long,
  title={Long-tailed classification by keeping the good and removing the bad momentum causal effect},
  author={Tang, Kaihua and Huang, Jianqiang and Zhang, Hanwang},
  journal={Advances in neural information processing systems},
  volume={33},
  pages={1513--1524},
  year={2020}
}

@inproceedings{zhou2020bbn,
  title={Bbn: Bilateral-branch network with cumulative learning for long-tailed visual recognition},
  author={Zhou, Boyan and Cui, Quan and Wei, Xiu-Shen and Chen, Zhao-Min},
  booktitle={Proceedings of the IEEE/CVF conference on computer vision and pattern recognition},
  pages={9719--9728},
  year={2020}
}

@inproceedings{li2020overcoming,
  title={Overcoming classifier imbalance for long-tail object detection with balanced group softmax},
  author={Li, Yu and Wang, Tao and Kang, Bingyi and Tang, Sheng and Wang, Chunfeng and Li, Jintao and Feng, Jiashi},
  booktitle={Proceedings of the IEEE/CVF conference on computer vision and pattern recognition},
  pages={10991--11000},
  year={2020}
}

@article{zhang2022self,
  title={Self-supervised aggregation of diverse experts for test-agnostic long-tailed recognition},
  author={Zhang, Yifan and Hooi, Bryan and Hong, Lanqing and Feng, Jiashi},
  journal={Advances in neural information processing systems},
  volume={35},
  pages={34077--34090},
  year={2022}
}

@inproceedings{zhou2017scene,
  title={Scene parsing through ade20k dataset},
  author={Zhou, Bolei and Zhao, Hang and Puig, Xavier and Fidler, Sanja and Barriuso, Adela and Torralba, Antonio},
  booktitle={Proceedings of the IEEE conference on computer vision and pattern recognition},
  pages={633--641},
  year={2017}
}

@article{hoover2000locating,
  title={Locating blood vessels in retinal images by piecewise threshold probing of a matched filter response},
  author={Hoover, AD and Kouznetsova, Valentina and Goldbaum, Michael},
  journal={IEEE Transactions on Medical imaging},
  volume={19},
  number={3},
  pages={203--210},
  year={2000},
  publisher={IEEE}
}

@inproceedings{ronneberger2015u,
  title={U-net: Convolutional networks for biomedical image segmentation},
  author={Ronneberger, Olaf and Fischer, Philipp and Brox, Thomas},
  booktitle={International Conference on Medical image computing and computer-assisted intervention},
  pages={234--241},
  year={2015},
  organization={Springer}
}

@inproceedings{long2015fully,
  title={Fully convolutional networks for semantic segmentation},
  author={Long, Jonathan and Shelhamer, Evan and Darrell, Trevor},
  booktitle={Proceedings of the IEEE conference on computer vision and pattern recognition},
  pages={3431--3440},
  year={2015}
}

@article{chen2017rethinking,
  title={Rethinking atrous convolution for semantic image segmentation},
  author={Chen, Liang-Chieh and Papandreou, George and Schroff, Florian and Adam, Hartwig},
  journal={arXiv preprint arXiv:1706.05587},
  year={2017}
}

@article{liu2021Swin,
  title={Swin Transformer: Hierarchical Vision Transformer using Shifted Windows},
  author={Liu, Ze and Lin, Yutong and Cao, Yue and Hu, Han and Wei, Yixuan and Zhang, Zheng and Lin, Stephen and Guo, Baining},
  journal={arXiv preprint arXiv:2103.14030},
  year={2021}
}

@inproceedings{xiao2018unified,
  title={Unified perceptual parsing for scene understanding},
  author={Xiao, Tete and Liu, Yingcheng and Zhou, Bolei and Jiang, Yuning and Sun, Jian},
  booktitle={Proceedings of the European conference on computer vision (ECCV)},
  pages={418--434},
  year={2018}
}

@article{bartlett2017spectrally,
  title={Spectrally-normalized margin bounds for neural networks},
  author={Bartlett, Peter L and Foster, Dylan J and Telgarsky, Matus J},
  journal={Advances in neural information processing systems},
  volume={30},
  year={2017}
}

@article{papyan2020prevalence,
  title={Prevalence of neural collapse during the terminal phase of deep learning training},
  author={Papyan, Vardan and Han, X. Y. and Donoho, David L.},
  journal={Proceedings of the National Academy of Sciences},
  volume={117},
  number={40},
  pages={24652--24663},
  year={2020},
  doi={10.1073/pnas.2015509117}
}

\newpage
\appendix

\section{Theoretical Analysis of Prior Estimation}
\label{app:theoretical_proof}

This appendix isolates the class-frequency effect of the one-way PEM objective in a simplified scalar model. The analysis is not intended as an exact description of a finite neural network. In particular, the quadratic term below regularizes the scalar output directly and should be viewed as a tractable surrogate for the effect of parameter regularization and finite optimization.

We consider a single PEM and the re-signed NPE output, so the relevant class coordinate is positive at the optimum. We further use a class-collapse approximation in which samples from class $c$ share a common scalar output $\eta_c$. This is motivated by the late-training within-class concentration associated with neural collapse~\cite{papyan2020prevalence}.

\subsection{Scalar objective}

Let $N_c$ be the number of training samples in class $c$. For a fixed class, consider
\begin{equation}
J_c(\eta)=-N_c\log\sigma(\eta)+\frac{\lambda}{2}\eta^2,
\qquad \lambda>0.
\end{equation}
The first term is the accumulated one-way logistic loss for that class and the second is the scalar regularizer. Its derivatives are
\begin{equation}
J_c'(\eta)=-N_c\sigma(-\eta)+\lambda\eta,
\end{equation}
and
\begin{equation}
J_c''(\eta)=N_c\sigma(\eta)\sigma(-\eta)+\lambda>0.
\end{equation}
Hence $J_c$ is strictly convex and has a unique minimizer.

\subsection{Dependence on class frequency}

\begin{proposition}[Frequency dependence and asymptotic behavior]
\label{prop:asymptotic}
Let $\eta_c^\ast$ be the minimizer of $J_c$. It is the unique positive solution of
\begin{equation}
\eta_c^\ast\bigl(1+e^{\eta_c^\ast}\bigr)=\frac{N_c}{\lambda}.
\label{eq:stationary_scalar}
\end{equation}
Consequently, $\eta_c^\ast$ is strictly increasing in $N_c$. Moreover, as $N_c/\lambda\to\infty$,
\begin{equation}
\eta_c^\ast
=
W\!\left(\frac{N_c}{\lambda}\right)+o(1)
=
\log\!\left(\frac{N_c}{\lambda}\right)
-
\log\log\!\left(\frac{N_c}{\lambda}\right)
+o(1),
\end{equation}
where $W$ is the principal Lambert $W$ function.
\end{proposition}

\begin{proof}
Setting $J_c'(\eta)=0$ gives
\[
\lambda\eta=\frac{N_c}{1+e^\eta},
\]
which is equivalent to Eq.~\ref{eq:stationary_scalar}. The function $\phi(\eta)=\eta(1+e^\eta)$ is strictly increasing for $\eta\ge0$, so the positive solution is unique and increases with $N_c/\lambda$.

Let $r=N_c/\lambda$. Equation~\ref{eq:stationary_scalar} gives $\eta e^\eta=r-\eta$. As $r\to\infty$, the solution also tends to infinity and $\eta/r\to0$, so $\eta e^\eta=r(1-o(1))$. Therefore $\eta=W(r)+o(1)$. The final expression follows from the standard asymptotic expansion $W(r)=\log r-\log\log r+o(1)$.
\end{proof}

\subsection{Relation to the empirical prior}

Because $N_c=Np(c)$,
\begin{equation}
\eta_c^\ast
=
\log p(c)+\log\!\left(\frac{N}{\lambda}\right)
-
\log\log\!\left(\frac{N_c}{\lambda}\right)
+o(1).
\end{equation}
The first class-dependent term is the desired log prior. The second term is common to all classes and therefore does not affect a softmax decision. The remaining $\log\log$ term varies with class frequency, so the scalar model predicts a slowly varying deviation from an exact log-prior.

For two classes $c$ and $d$,
\begin{equation}
\eta_c^\ast-\eta_d^\ast
=
\log\frac{N_c}{N_d}
-
\log\frac{\log(N_c/\lambda)}{\log(N_d/\lambda)}
+o(1).
\end{equation}
Thus the leading difference is the log prior ratio, while the residual changes more slowly with the counts. This is the sense in which the one-way objective provides a prior-related signal. The empirical analysis below tests how closely the learned network follows this idealized behavior.

\subsection{Empirical Results}
\label{app:empirical_result}

The relationship between the NPE estimated prior 
$\hat{p}^{\mathrm{val}}(y)$, the empirical class prior $p(y)$, and the representation 
confidence $\mathrm{Conf}(y)$ is evaluated on CIFAR-100-LT with imbalance ratio 
$\rho = 100$. Here, $\hat{p}^{\mathrm{val}}(y)$ denotes the class-wise normalized NPE signal reported on the validation set, while $\mathrm{Conf}(y)$ denotes the class-wise representation-confidence statistic used for this diagnostic analysis. These quantities are used only for the analysis in this appendix and do not enter NPE training or inference.  
Results are shown in Figure~\ref{fig:prob}.

\begin{figure}[htb]
	\begin{center}
		\includegraphics[width=0.7\linewidth]{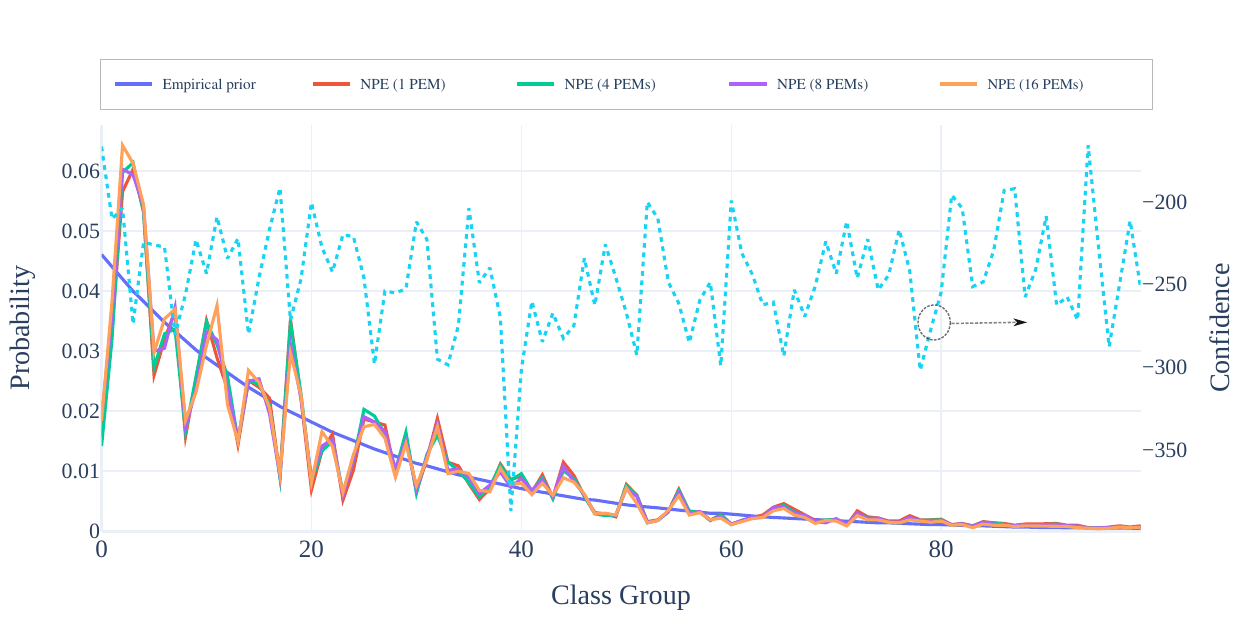}
\caption{
	NPE estimated prior $\hat{p}^{\mathrm{val}}(y)$ and empirical class prior $p(y)$ (left axis), and representation confidence $\mathrm{Conf}(y)$ (right axis), on CIFAR-100-LT ($\rho = 100$).
}
		\label{fig:prob}
	\end{center}
\end{figure}

Three consistent trends emerge:

\begin{enumerate}
	\item \textbf{Global Prior Fidelity.}  
	The NPE estimate follows the monotonic decay of the empirical prior 
	$p(y)$, showing that the estimate $\eta_y$ tracks the main class-frequency trend.  
	This empirical behavior aligns with the theoretical analysis in Appendix~\ref{app:theoretical_proof}.
	
	\item \textbf{Estimator Stability.}  
	The values of $\hat{p}^{\mathrm{val}}(y)$ are stable across different ensemble 
	sizes 
	$N_{\mathrm{PEM}} \in \{1,4,8,16\}$, indicating that the overall frequency trend is not strongly dependent on the number of PEMs in this experiment.
	
	\item \textbf{Geometry--Prior Interaction.}  
	A clear relationship appears between representation confidence and the 
	estimation residual
	\begin{equation}
		E(y) = \hat{p}^{\mathrm{val}}(y) - p(y).
	\end{equation}
	Classes with low confidence (high feature dispersion) tend to have 
	$E(y) > 0$, meaning the NPE \emph{overestimates} their prior.  
	Conversely, compact classes (high confidence) tend to be slightly 
	underestimated.  
	This arises because dispersed classes require larger PEM output 
	magnitudes to satisfy the one-way logistic objective, which the NPE 
	implicitly interprets as higher frequency.
\end{enumerate}

\paragraph{Relation to the analysis.}
The scalar model predicts a monotone frequency signal with a logarithmic leading term, not exact prior recovery. The network follows this trend but also shows class-dependent residuals. The correlation with representation confidence suggests that finite feature geometry contributes to those residuals in addition to the count-dependent effect captured by the scalar analysis.

\end{document}